\documentclass[letterpaper, 10 pt, journal, twoside]{ieeetran}
\usepackage{xcolor}
\usepackage{cite}
 
\usepackage{amsmath,amsfonts}
\usepackage{amssymb}
\usepackage{mathtools}
\usepackage{dsfont}

\usepackage{textcomp}
\usepackage{stfloats}
\usepackage{url}
\usepackage{siunitx}
\usepackage{amsthm}
\usepackage{verbatim}

\usepackage{graphicx}
\usepackage[caption=false,font=normalsize,labelfont=sf,textfont=sf]{subfig}
\usepackage{makecell}
\usepackage{booktabs}
\usepackage{array}

\usepackage{algorithm}
\usepackage[noend]{algpseudocode}

\usepackage{balance}

\usepackage{mysymbol}
\usepackage{hyperref}
\def\BibTeX{{\rm B\kern-.05em{\sc i\kern-.025em b}\kern-.08em
    T\kern-.1667em\lower.7ex\hbox{E}\kern-.125emX}}
\usepackage{balance}

\newtheorem{theorem}{Theorem}[section]

\newtheorem{assumption}[theorem]{Assumption}
\newtheorem{remark}[theorem]{Remark}
\newtheorem{problem}[theorem]{Problem}
\newtheorem{proposition}[theorem]{Proposition}

\newcommand{\R}{{\mathbb{R}}}

\begin{document}
\title{Safe Planning in Unknown Environments \\ Using Conformalized Semantic Maps}
\author{David Smith Sundarsingh$^1$, Yifei Li$^1$, Tianji Tang$^2$, George J. Pappas$^3$, Nikolay Atanasov$^2$ and Yiannis Kantaros$^1$
\thanks{Manuscript created September 29, 2025; Revised December 28, 2025; Accepted February 3, 2026. This paper was recommended for publication by Editor Lucia Pallottino upon evaluation of the Associate Editor and Reviewers' comments. This work was supported by
the ARL grant DCIST CRA W911NF-17-2-0181 and NSF CAREER award CNS \#2340417.}\thanks{$^1$The authors are with the Department of Electrical and Systems Engineering, Washington University in St. Louis, MO. 63108, USA. Emails: \{d.s.sundarsingh,liyifei,ioannisk\}@wustl.edu}.
\thanks{$^2$The authors are with the Department of Electrical and Computer Engineering, University of California, San Diego, CA. 92093, USA. Emails: \{tit006,natanasov\}@ucsd.edu}.
\thanks{$^3$The author is with the Department of Electrical and Systems Engineering, University of Pennsylvania, Philadelphia, PA. 19104, USA. Email: pappasg@seas.upenn.edu.}}

\markboth{IEEE ROBOTICS AND AUTOMATION LETTERS. PREPRINT VERSION. ACCEPTED FEBRUARY, 2026}%
{Sundarsingh \MakeLowercase{et al.}: Conformal Safe Semantic Planning}

\maketitle
\begin{abstract}
    This paper addresses semantic planning problems in unknown environments under perceptual uncertainty. The environment contains multiple unknown semantically labeled regions or objects, and the robot must reach desired locations while maintaining class-dependent distances from them. We aim to compute robot paths that complete such semantic reach-avoid tasks with user-defined probability despite uncertain perception. Existing planning algorithms either ignore perceptual uncertainty—thus lacking correctness guarantees—or assume known sensor models and noise characteristics. In contrast, we present the first planner for semantic reach-avoid tasks that achieves user-specified mission completion rates without requiring any knowledge of sensor models or noise. This is enabled by quantifying uncertainty in semantic maps—constructed on-the-fly from perceptual measurements—using conformal prediction in a model- and distribution-free manner. We validate our approach and the theoretical mission completion rates through extensive experiments, showing that it consistently outperforms baselines in mission success rates.
\end{abstract}

\section{Introduction}\label{sec:intro}

\IEEEPARstart{E}{xtensive} research on autonomous robot navigation has resulted in numerous planners capable of designing trajectories that reach known goal regions while avoiding obstacles \cite{a-star,karaman2011sampling}.
Recent advances in AI-enabled robot vision and semantic mapping have created new opportunities to move beyond these well-studied geometric planning approaches toward semantic planning, where task completion requires reasoning not only about the geometric structure of the environment (e.g., obstacle locations) but also its semantic structure (e.g., object categories or terrain types) \cite{gionfrida2024computer,placed2023survey}.

This paper focuses on semantic planning problems in environments populated with multiple objects and regions whose locations and semantic categories are initially unknown. The robot’s goal is to reach desired destinations while maintaining category-dependent safety distances from these semantic regions and objects. Recent planning algorithms tackle such tasks by leveraging AI-enabled perception systems (e.g., object recognition algorithms) to collect geometric and semantic observations 
\cite{kantaros2022perception,fu2016optimal,vasilopoulos2020reactive,malczyk2025semantically,pal2021learning,hierarchical3DSG}. These approaches typically either (i) employ deep learning methods to directly map perceptual inputs to robot actions (e.g., \cite{georgakis2022uncertainty,malczyk2025semantically,hierarchical3DSG}) or (ii) leverage semantic SLAM techniques  
\cite{bowman2017probabilistic,ssmi_temp,tian2022kimera} to build semantic maps that are used by downstream planners (e.g., \cite{kantaros2022perception,fu2016optimal,psomiadis2025hcoa}). Despite impressive empirical performance, in these architectures the planning layer generally treats outputs from the perception or mapping modules as ground truth, ignoring their inherent uncertainty. As a result, these approaches lack mission completion guarantees. 
%
Attempts to model perceptual uncertainty within semantic and geometric planning have been investigated in \cite{kantaros2022perception,fu2016optimal,latentBKI} and \cite{renganathan2023risk,bry2011rapidly,palmieri2017kinodynamic,ho2022gaussian}, respectively. However, these methods often depend on unreliable proxies such as SLAM-derived confidence scores that do not represent true probabilities \cite{minderer2021revisiting}, or assume known sensor models and Gaussian noise—assumptions that may not hold in practice. 

\begin{figure}[t]
    \centering
    \vspace{-0.5cm}
\includegraphics[width=\linewidth]{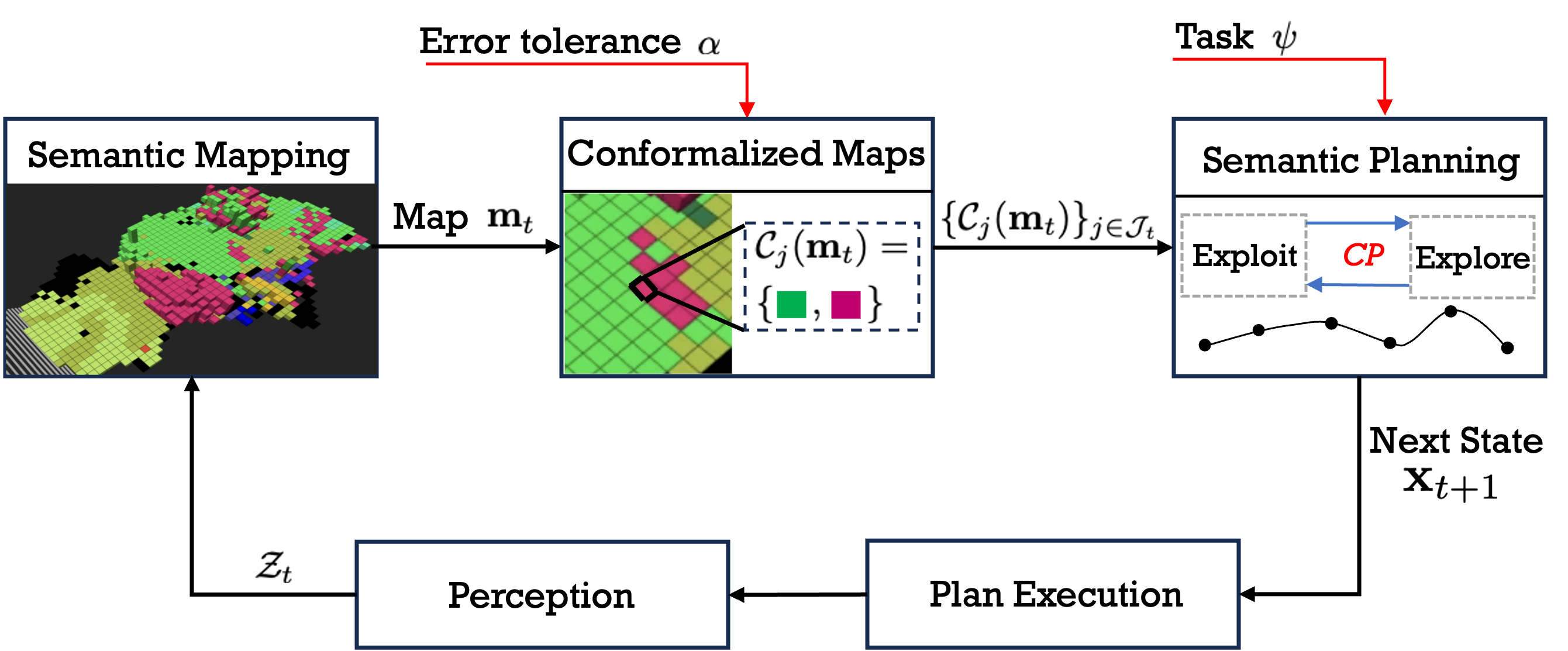}
    \caption{Illustration of the proposed planning framework for semantic reach-avoid tasks $\psi$. Given categorical and range measurements, a semantic map is constructed \cite{ssmi_temp}. To account for map uncertainty caused by perceptual imperfections, conformal prediction is employed to generate sets of maps that contain the ground-truth environment with a user-defined probability. These sets are then used by a planner to generate uncertainty-aware paths that complete the assigned tasks with a user-defined probability $1-\alpha$.}
    \vspace{-0.5cm}
    \label{fig:overview}
\end{figure}

This paper addresses these limitations by developing a semantic planner for unknown environments that completes semantic reach-avoid tasks with a user-defined probability under perceptual uncertainty, while remaining agnostic to sensor models and noise characteristics. 
Robots are equipped with imperfect perception systems (e.g., cameras with semantic segmentation) that provide geometric and semantic observations. Our framework consists of three main components; see Fig. \ref{fig:overview}.
First, we leverage existing semantic mapping algorithms that capture dense semantic and geometric environmental information and generate a multi-class semantic map of the environment \cite{ssmi_temp}.
Second, we quantify uncertainty of the generated semantic maps using Conformal Prediction (CP), a statistical, model- and distribution-free tool \cite{CPbase}. This produces a set of maps guaranteed to contain the ground truth with a user-defined probability. 
Third, we develop a semantic planner that balances exploitation and exploration. In exploitation mode, the planner selects the most conservative map from the CP set and generates paths using any search- or sampling-based planner \cite{a-star,karaman2011sampling}. 
When uncertainty is too high for a feasible solution to exist, the planner switches to exploration mode to reduce uncertainty (modeled by the size of the CP set) until a feasible path can be found. 
We show both theoretically and empirically that our method achieves desired mission completion rates. We also conduct experiments highlighting the exploitation–exploration trade-offs as the required success probability increases. Our comparative experiments also demonstrate that our approach outperforms baselines that either treat mapping outputs as ground truth or rely on uncertainty metrics provided by the mapping algorithm.

\textbf{Related Work:} \textit{(i) Perception-based Planning:} Numerous planners for unknown environments have been proposed. However, these either neglect perceptual/mapping uncertainty and, thus, lack correctness guarantees 
\cite{ray2024task,vasilopoulos2020reactive,malczyk2025semantically,georgakis2022uncertainty,pal2021learning,hierarchical3DSG}, or assume known sensor models and noise characteristics to design probabilistically correct paths \cite{kantaros2022perception,fu2016optimal,renganathan2023risk,bry2011rapidly,palmieri2017kinodynamic,ho2022gaussian}. In contrast, our planner achieves user-specified mission completion rates without assuming any knowledge of sensor models or noise. The closest work to ours is \cite{PwC}, which employs CP to quantify the uncertainty of object detectors and design probabilistically correct paths. However, \cite{PwC} focuses solely on geometric planning and cannot straightforwardly handle combined geometric and semantic uncertainty.  \emph{(ii) CP for Verified Autonomy:} {CP has been applied to quantify uncertainty in various AI-enabled components (e.g., perception, motion prediction, and LLMs) \cite{lindemann2023safe,sheng2024safe,cpPerp1,ren2023robots,wang2024probabilistically,lekeufack2024conformal,sun2023conformal}.}
%
%
Inspired by these works, we propose the first use of CP to quantify semantic map uncertainty enabling probabilistically correct semantic planning.


\textbf{Contributions}: 
First, we introduce the first semantic planner for unknown environments that achieves user-defined mission completion rates under perceptual uncertainty, while remaining agnostic to sensor models and noise characteristics. Second, we present the first application of CP to quantify the uncertainty of semantic maps generated by existing mapping algorithms, also in a model-agnostic manner. Third, we present experiments that validate the theoretical mission completion guarantees and show that our method consistently outperforms baselines in terms of mission success rates. 
\section{Problem Statement}\label{sec:problem}
\subsection{Robot \& Environment Modeling}\label{sec:sysAndEnv}


Consider a robot with configuration $\bbx_t \in \ccalX$ at discrete time $t \geq 0$. We assume that the configuration $\bbx_t$ is known (e.g., from encoders and localization). The robot operates in an environment $\Omega\subseteq\R^d$, $d\in\{2,3\}$, which is represented as a collection of disjoint sets $\ccalL_k\subseteq\Omega$, each associated with a semantic category $k\in\mathcal{K}=\{0,1,\dots,K\}$ (e.g., cars, walls, trees). Here, $\ccalL_0$ denotes free space, and $\ccalL_k$ are fixed over time and unknown to the robot.

\subsection{Reachability Task with Semantic Safety Constraints}\label{sec:safetySpec}
The robot is tasked with reaching a known set of goal states $\ccalX_g\subseteq\ccalX$ while maintaining class-dependent safe distances $d_k$ from regions $\ccalL_k$. We denote the assigned task by $\psi$. As an example, consider a safety specification requiring the robot to maintain a distance of $1\text{m}$ from $\ccalL_\text{tree}$ and $3\text{m}$ from  $\ccalL_\text{car}$ while navigating towards a desired destination. Such semantic reach-avoid tasks can be completed by following a finite-horizon path $\bbx_{0:H}$, defined as a sequence of robot states, i.e., $\bbx_{0:H}=\bbx_0,\dots,\bbx_t,\dots,\bbx_H$, for some horizon $H$, where (i) $\bbx_H\in\ccalX_g$, and (ii) all $\bbx_t$ satisfy the safety constraints, i.e.,
\begin{align}\label{equ:pathCorrectness}
    \min_{\bbq\in \ccalL_k}\lVert\bbx_t|_{\Omega}-\bbq\rVert_2\geq d_k,\quad \forall t\in\{0,\dots,H\},
\end{align}{for all $\ccalL_k\subseteq\Omega$, where $k\in\mathcal{K}\setminus\{0\}$.}
In \eqref{equ:pathCorrectness}, $\bbx_t|_{\Omega}$ denotes the robot's position, i.e., the projection of the robot state onto $\Omega$. Completion of $\psi$ using a robot path $\bbx_{0:H}$ is denoted by $\bbx_{0:H}\models\psi$, where $\models$ denotes the satisfaction relation.
 
\subsection{Perception System and Semantic Mapping}\label{sec:mappingAlgo}


\begin{figure}[t]
    \centering
\includegraphics[width=\linewidth]{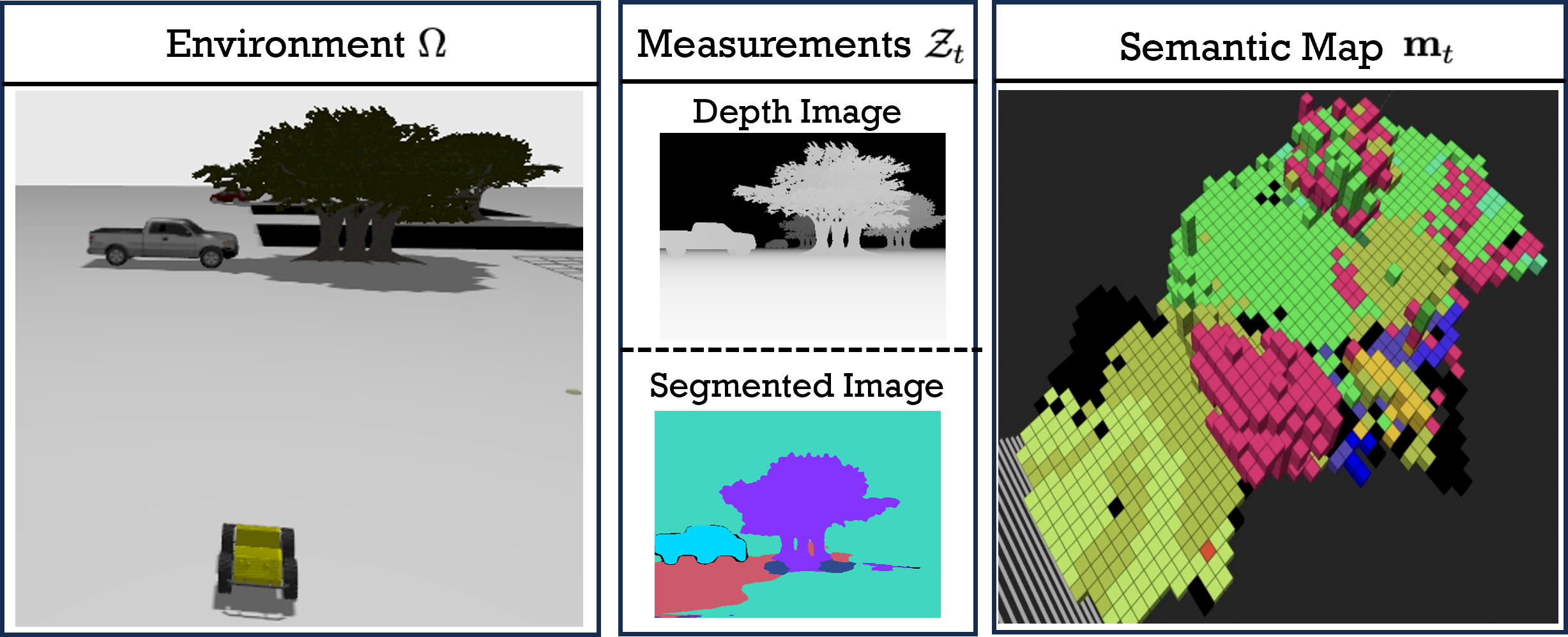}
\vspace{-0.5cm}
    \caption{The mapping algorithm uses depth and categorical measurements $\ccalZ_t$, provided as a depth image and a segmented image, respectively, to generate a 3D semantic grid map $\bbm_t$. Each grid cell is associated with a semantic category, and its color corresponds to that category.}
    \vspace{-0.5cm}
    \label{fig:envmapmeas}
\end{figure}
The robot is equipped with a sensor of range $r_{\text{max}}>0$ that provides information about the distances to, and semantic categories of, surrounding objects along a set of rays. We denote by $\mathcal{Z}_t:=\{z_{t,b}\}_b$ the set of measurements collected along ray $b$ at time $t$, where each measurement is given by $z_{t,b}=\{r_{t,b},y_{t,b}\}$, with $r_{t,b}\in\mathbb{dR}_{\geq0}$ and $y_{t,b}\in\ccalK$ denoting the range and categorical measurements, respectively. 

These labeled range measurements, {obtained by processing RGB-D images,} can be fused online to build a semantic representation of $\Omega$ using semantic mapping algorithms. We adopt \cite{ssmi_temp}, which models the environment as a multi-class grid map $\bbm$ defined as a collection of cells $j\in\ccalJ=\{1,\dots,M\}$, each labeled with a category $m_j\in\ccalK$ and represented by its center $c_j\in\Omega$; see Fig. \ref{fig:envmapmeas}. 
%
To account for perceptual noise, each cell $m_j$ maintains a probability mass function (PMF)
%
$p_t(m_j)=p(m_j|\bbx_{0:t},\ccalZ_{0:t})$ over $\mathcal{K}$, conditioned on the measurements $\ccalZ_{0:t}$ collected along the robot path $\bbx_{0:t}$. If at some $t$, no measurement corresponds to cell $j$, $p_t(m_j)=p_{t-1}(m_j)$. 
Accordingly, the distribution of map is defined as $p_{t}(\bbm)=\prod_{j\in\ccalJ}p_{t}(m_j)$. For brevity, we denote by $\sigma$ a mapping algorithm that takes as input the distributions $p_{t-1}(m_j)$ and the new measurement $\ccalZ_{t}$ collected at $\bbx_{t}$, and outputs the updated distributions $p_{t}(m_j),\forall j\in\ccalJ_t$, i.e., $p_{t}(\bbm)=\sigma(\Omega,p_{t-1}(\bbm),\ccalZ_{t})$, {where $\ccalJ_t\subseteq\ccalJ$ is the set of indices of the grid cells that have been detected till time $t$}; map update at $t=0$ occurs using a prior $p_{-1}(\bbm)$.  
Hereafter, when it is clear from the context, we refer to the random variable $\bbm_t$ without explicitly specifying its density $p_t(\bbm)$.

\subsection{Problem Statement}
Consider any user-defined initial robot state $\bbx_0$ and semantic reach-avoid task $\psi$ for an environment $\Omega$. We assume that $\Omega$ is sampled from an unknown distribution $\ccalD$, from which a finite set of i.i.d environments can be sampled.
The robot does not know $\Omega$; instead it has access to a prior map $\bbm_{-1}$ that can be updated online using the mapping algorithm $\sigma$. We aim to design a planning algorithm that, given an environment sampled from $\ccalD$, can generate a robot path $\bbx_{0:H}$, for some $H>0$, that completes $\psi$ with a user-specified probability $1-\alpha$, given the sequence of maps $\bbm_{0:H}$ updated online. Formally, we address the following problem: 


\begin{problem}\label{prob:main}
{Consider an unknown environment $\Omega \sim \ccalD$ with prior information $\bbm_{-1}$, and any initial robot state $\bbx_0$ and semantic reach-avoid task $\psi$ defined over $\Omega$.} Given a mapping algorithm $\sigma$, design a planning algorithm that generates a robot path $\bbx_{0:H}$, for some $H>0$, ensuring that 
\begin{equation}\label{eq:pr}
    \mathbb{P}_{{\Omega\sim \ccalD}}(\bbx_{0:H}\models\psi~|~\bbm_{0:H})\geq 1-\alpha,
\end{equation}
where $1-\alpha\in(0,1)$ is a user-specified success rate.
\end{problem}

We note that the probabilistic task completion statement in \eqref{eq:pr} accounts for errors at the mapping level, but not necessarily at the sensor level. For instance, certain sensor rays $b$ may fail to return a measurement from an object occupying cell $j$ (e.g., due to malfunctions). In such cases, the PMF of $m_j$ is never updated, creating the false impression that it lies in free space. The guarantee in \eqref{eq:pr} does not cover these sensor-level errors; rather, it addresses labeling and mapping errors in cells $m_j$ for which at least one associated measurement has been collected. 
Failures can also occur due to objects that (i) are close enough to the robot, leading to violations of \eqref{equ:pathCorrectness} but they (ii) lie outside the sensor’s field of view. To exclude such sensor-based errors, we make the following assumption.

\begin{assumption}[Sensors]\label{assume:sensors}
(i) The robot is equipped with an \textit{omnidirectional} sensor of range $r_{\text{max}}$ generating $\mathcal{Z}_t$. (ii) The sensor range $r_{\text{max}}$ satisfies $r_{\text{max}}\geq d_k$ for all $k\in\ccalK$ 
(iii) Any cell $j$ in the map within the sensor's field of view always has a measurement associated with it. Due to (iii), if no measurements are obtained for cell $j$ (lying within the sensor range), then it corresponds to free-space.
\end{assumption}




\section{Safe Semantic Planning in Unknown Environments
}\label{sec:method}


A natural approach to Problem \ref{prob:main} is to plan directly over the distribution of the map $\bbm_{t}$, as in~\cite{fu2016optimal,kantaros2022perception}. However, in practice, the  sensor models used to update the map distribution are inaccurate, so the PMFs $p_t(m_j)$ do not represent true probabilities. As a result, planning directly over $\bbm_{t}$ may compromise the probabilistic correctness of the generated paths. To overcome this limitation, we propose a new planner that makes no assumptions about the sensor model and noise characteristics. In Section \ref{sec:calibration}, we employ conformal prediction (CP) to ‘calibrate’ $\bbm_t$ by computing prediction sets for each cell $j\in\ccalJ_t$ that contain the true label with user-defined confidence. These sets are then used in Section \ref{sec:safePlan} to design paths $\bbx_{0:H}$ that solve Problem \ref{prob:main}.


\subsection{Conformalizing Semantic Maps}\label{sec:calibration}

As discussed earlier, the distribution of $\bbm_{t}$ does not represent true probabilities, compromising the correctness of paths designed only using $\bbm_{t}$. To address this, for any {test environment $\Omega_{\text{test}}\sim\ccalD$}, we aim to design sets $\ccalC_j(\bbm_t)\subseteq\ccalK$ of labels for each cell $j\in\ccalJ_t$ detected along a path $\bbx_{0:H}$ and time step $t\in\{0,\dots,H\}$ satisfying the following property:
\begin{align}\label{eq:predSetProperty}
\mathbb{P}_{{\Omega\sim\ccalD}}\left(C_{0:H}\mid\bbm_{0:H}\right)\geq 1-\alpha,
\end{align}
where $C_{0:H}=\bigwedge_{\forall j\in\ccalJ_t,t\in[0:H]}(k_j^{\text{gt}}\in\ccalC_j(\bbm_t))$ and $k_j^{\text{gt}}\in\ccalK$ is the true label of the region occupying cell $j$. 
In words, \eqref{eq:predSetProperty} guarantees the inclusion of the true labels $k_j^{\text{gt}}$ in the sets $\ccalC_j(\bbm_t)$,$\forall j\in\ccalJ_t$ in $\bbm_t$, and $\forall t\in\{0,\dots,H\}$, and any robot path $\bbx_{0:H}$ followed to construct $\bbm_{0:H}$.

To construct these sets for ${\Omega_{\text{test}}}$, we employ CP to quantify uncertainty in the outputs of the mapping algorithm $\sigma$. Applying CP requires (i) a set of calibration {environments} ${\{\Omega_i\}_{i=1}^D}\sim\ccalD$ with known ground-truth labels $k_j^{\text{gt}}$ for all $j\in\ccalJ$, 
and (ii) a non-conformity score (NCS) that captures the error of $\sigma$, which generates maps $\bbm_{0:H}$ along \textit{any} path $\bbx_{0:H}$. It is crucial that the NCS accounts for mapping errors along arbitrary robot paths, rather than a single fixed path, since otherwise \eqref{eq:predSetProperty} would hold only for that path. Thus, we adopt the following NCS, which is path (and, consequently, task)-agnostic  and captures the worst-case mapping error, associated with the calibration {environment $\Omega_i$}:
\begin{align}\label{equ:NCS}
s_i=\max_{\bbx_{0:H}\in\ccalP;j\in\ccalJ_t; t\in[0,H]}\Big(1-p_t(m_j=k_j^{\text{gt}})\Big),
\end{align}
where $\ccalP$ is the set of all possible finite paths in $\Omega_i$. In \eqref{equ:NCS}, the term $p_t(m_j=k_j^{\text{gt}})$ denotes the probability that cell $j$ has the ground-truth label according to the map $\bbm_t$. 
Consequently, $1-p_t(m_j=k_j^{\text{gt}})$ models the mapping error at cell $j$ after collecting measurements along a specific path $\bbx_{0:t}$. The maximum operator in \eqref{equ:NCS} computes the worst-case error, i.e., the largest error across all cells $j\in\ccalJ_t$, all time steps $t\in[0,H]$, and all paths $\bbx_{0:H}\in\ccalP$.\footnote{Computing \eqref{equ:NCS} exactly requires evaluating the maximum over an infinite number of paths, which is impractical. As in \cite{PwC}, in practice the NCS is approximated using a finite set of paths $\ccalP$; see Section~\ref{sec:setup} for details.}
We compute the NCSs $\{s_i\}_{i=1}^D$ for all calibration {environments} and determine their $\frac{\lceil(D+1)(1-\alpha)\rceil}{D}$-quantile, denoted by $\hat{s}$.

For an unseen test {environment $\Omega_{\text{test}}\sim\ccalD$} with unknown true labels for cells $j$, at each step $t$ we compute the following prediction set of labels $\forall j\in\ccalJ_t$ (see Alg. \ref{algo:uncertaintyMap}):
\begin{align}\label{equ:gridPredSet}
    \ccalC_j(\bbm_t)=\{k\in\ccalK|p_t(m_j=k)\geq 1-\hat{s}\},
\end{align}
collecting all labels $k\in\ccalK$ satisfying $p_t(m_j=k)\geq 1-\hat{s}$ where $p_t$ refers to the PMF computed by $\sigma$ [line \ref{algo:createPredSet}, Alg. \ref{algo:uncertaintyMap}].
In Section \ref{sec:guarantees}, we show that these prediction sets satisfy \eqref{eq:predSetProperty}.

\begin{algorithm}[t]
\footnotesize
\caption{Conformalization of Semantic Map}\label{algo:uncertaintyMap}
\begin{algorithmic}[1]
\State \textbf{Input}: Map $\bbm_t$, score quantile $\hat{s}$, 
semantic classes $\ccalK$
\For{$j\in\ccalJ_t$}
\State $\ccalC_j(\bbm_t)=\{k\in\ccalK|p_t(m_j=k)\geq 1-\hat{s}\}$\label{algo:createPredSet}
\EndFor
\Return $\{\ccalC_j(\bbm_t)\}_{j\in\ccalJ_t}$\label{algo:return}
\end{algorithmic}
\end{algorithm}


\subsection{Planning over Conformalized Semantic Maps}\label{sec:safePlan}


{Consider an environment $\Omega\sim\mathcal{D}$. Our goal is to compute a plan $\bbx_{0:H}$ {satisfying any user-defined mission specification $\psi$} with probability at least $1-\alpha$. This path will be constructed incrementally as the map $\bbm_t$ is updated online; see Alg. \ref{algo:planner}. First, we collect the calibration dataset {$\{\Omega_i\}_{i=1}^D$}, and compute $\hat{s}$ as described in Section \ref{sec:calibration} [lines \ref{algo:collectData}-\ref{algo:computeScore}, Alg. \ref{algo:planner}].
%
Suppose at time $t>0$, the robot is at state $\bbx_{t}$, collects measurements $\ccalZ_t$, and constructs the map $\bbm_t$ [lines \ref{algo:measure}-\ref{algo:mapping}, Alg. \ref{algo:planner}]. Using $\bbm_t$, the prediction sets $\ccalC_j(\bbm_t)$ are constructed for all cells $j\in\ccalJ_t$ [line \ref{algo:conformalize}, Alg. \ref{algo:planner}]. We now discuss how the next robot state $\bbx_{t+1}$ is computed using these sets [lines \ref{algo:plan}–\ref{algo:timeUpdate}, Alg. \ref{algo:planner}]. The next state $\bbx_{t+1}$ must be chosen so that it satisfies \eqref{equ:pathCorrectness}. Since the regions $\ccalL_k$ are unknown, satisfaction of \eqref{equ:pathCorrectness} must be checked using the map $\bbm_t$. Specifically, given $\bbm_t$, \eqref{equ:pathCorrectness} can be re-written as:
\begin{equation}\label{equ:newSafetyConstraint}
    \lVert\bbx_{t+1}|_\Omega-c_j\rVert \geq d_{k_j}+r_m,
\end{equation}
for all cells $j\in\ccalJ_t$, where $r_m$ is the grid resolution, $c_j\in\Omega$ is the center of cell $j$ and $k_j$ is the label assigned to cell $j$.\footnote{The grid resolution is added to ensure that the robot never gets close to an object even if the object occupies the entirety of the grid cell.} 
\textcolor{black}{Note that cells $j\in\ccalJ\setminus\ccalJ_t$ within the sensor range $r_{\text{max}}$ are free space due to Assumption~\ref{assume:sensors}. }
The next state $\bbx_{t+1}$ can be chosen using any existing planner if the true label $k_{j}^{\text{gt}}$ of each cell $j$ were known to the robot, which is not the case here. 
To address this, we leverage sets $\ccalC_j(\bbm_t)$, constructed as shown in \eqref{equ:gridPredSet}, that contain the true label with user-defined confidence. From each set, we select the most conservative label, i.e., $k_j=\argmax_{k\in\ccalC_j(\bbm_t)} d_k$, for every cell $j\in\ccalJ_t$. \textcolor{black}{All cells $j\in\ccalJ\setminus\ccalJ_t$ outside the sensor range are considered part of the free space i.e., $k_j=0$, for all $j\in\ccalJ\setminus\ccalJ_t$.} Given this assignment of labels, the planner adaptively switches between exploitation and exploration modes as follows. 

\textbf{Exploitation Phase:} Using an existing motion planner \cite{a-star,karaman2011sampling,janson2015fast}, we first check whether there exists a robot path $\bbp=\bbp_0,\bbp_1,\dots,\bbp_T$, with $\bbp_n\in\ccalX$ for all $n\in\{0,\dots,T\}$ and some $T\geq 0$ that completes $\psi$ under the assigned labels $k_j$ [line \ref{algo:plan}, Alg. \ref{algo:planner}]. The path must satisfy $\bbp_0=\bbx_t$, $\bbp_T\in\ccalX_g$, and $\lVert \bbp_{n}|_\Omega-c_j\rVert \geq d_{k_j}+r_m,\forall n$. If such a feasible plan exists, the robot selects $\bbx_{t+1}=\bbp_1$, moves towards it, updates the time step to $t+1$, and repeats the process [lines \ref{algo:assign}-\ref{algo:timeUpdate}, Alg. \ref{algo:planner}].

\textbf{Exploration Phase:} If no feasible solution is found during the exploitation phase, the planner switches to exploration [line \ref{algo:ExploreTrigger}, Algo. \ref{algo:planner}]. In this mode, any exploration algorithm can be employed to generate paths $\bbp$ that optimize information-theoretic objectives (e.g., mutual information) subject to the constraints in \eqref{equ:newSafetyConstraint}; see, e.g., \cite{ssmi_temp,explore}. Given such a path $\bbp$, the planner selects $\bbx_{t+1}=\bbp_1$, sends it to the robot, updates the time step to $t+1$, and repeats the process [lines \ref{algo:plan}-\ref{algo:timeUpdate}, Alg. \ref{algo:planner}]. The goal is for exploration paths to reduce uncertainty (i.e., prediction set sizes) for the robot to re-enter exploitation mode.

{During both exploration and exploitation we also ensure that $\lVert\bbx_{t+1}-\bbx_t\rVert<r_{\text{max}}$ to eliminate cases where the next state may lie in the unexplored part of the map. 
} The above process terminates once $\bbx_{t+1}\in\ccalX_g$, {and the robot reaches $\bbx_{t+1}$}. 
Concatenation of these states across time yields $\bbx_{0:H}$.

\begin{algorithm}[t]
\footnotesize
\caption{Safe Planning Over Conformalized Maps}\label{algo:planner}
\begin{algorithmic}[1]
\State \textbf{Input}: Prior map $\bbm_{-1}$; Task $\psi$; Initial state $\bbx_0$; Set of labels $\ccalK$
\State Initialize $t=0$
\State {Collect calibration environments $\{\Omega_i\}_{i=1}^D$}\label{algo:collectData}
\State {Compute NCSs $\{s_i\}_{i=1}^D$ and quantile $\hat{s}$}\label{algo:computeScore}
\While{$\mathbf{x}_t\not\in\ccalX_g$}
\State $\ccalZ_t=\texttt{Measure}(\Omega,\bbx_t)$\label{algo:measure}
\State $\bbm_t=\sigma(\Omega,\bbm_{t-1},\ccalZ_t)$\label{algo:mapping}
\State $\left\{\ccalC_j(\bbm_{t})\right\}_{j\in\ccalJ_t}=\texttt{Conformalize}(\bbm_{t},\hat{s},\ccalK)$ (see Alg. \ref{algo:uncertaintyMap})\label{algo:conformalize} 
\State $\bbp=\texttt{Plan}\left(\bbx_t,\ccalX_g,\{\ccalC_j(\bbm_{t})\}_{j\in\ccalJ_{t}}\right)$\label{algo:plan}
%
\If{$\bbp=\emptyset$}\label{algo:ExploreTrigger}
\State $\bbp=\texttt{Explore}\left(\bbx_t,\{\ccalC_j(\bbm_{t})\}_{j\in\ccalJ_{t}}\right)$\label{algo:explore}
\EndIf
\State $\bbx_{t+1}=\bbp(1)$\label{algo:assign}
\State $\texttt{Execute}(\bbx_{t+1})$\label{algo:execute}
\State $t=t+1$\label{algo:timeUpdate}
\EndWhile
\end{algorithmic}
\end{algorithm}

\section{Performance Guarantees}\label{sec:guarantees}

In this section, we show that our planner addresses Problem \ref{prob:main}, i.e., the path $\bbx_{0:H}$ generated by Algorithm \ref{algo:planner} completes $\psi$ with probability $1-\alpha$. 
This result builds upon the fact that the sets $\ccalC_j(\bbm_t)$, constructed by Alg.~\ref{algo:uncertaintyMap}, satisfy \eqref{eq:predSetProperty}, i.e., they jointly contain the true labels for all $j\in\ccalJ_t$, $t\in[0,H]$ and any path $\bbx_{0:H}$ with probability at least $1-\alpha$. 

\begin{proposition}[Mapping Guarantees]\label{prop:perception}
    Consider a test {environment $\Omega_{\text{test}}\sim\ccalD$} with unknown true labels $\{k_j^{\text{gt}}\}_{j\in\ccalJ}$. Given any path $\bbx_{0:H}$ in $\Omega_{\text{test}}$ yielding a sequence of maps $\bbm_{0:H}$, the prediction sets $\{\ccalC_j(\bbm_t)\}_{j\in\ccalJ_t}$, computed at every $t\in\{0,\ldots,H\}$ constructed by Alg. \ref{algo:uncertaintyMap} satisfy \eqref{eq:predSetProperty}, for a user-defined $\alpha\in(0,1)$.
\end{proposition}

\begin{proof}
Using a standard CP argument \cite{PwC,angelopoulos2023conformal}, we have:
\begin{equation}\label{eq:CP1}
\mathbb{P}_{{\Omega_{\text{test}}\sim\ccalD}}\left(s_{\text{test}}\leq\hat{s}\right|\bbm_{0:H})\geq1-\alpha,
\end{equation}
where $s_{\text{test}}$ is the NCS associated with {$\Omega_{\text{test}}$}. Note that $s_{\text{test}}$ is unknown as the ground truth labels $k_j^{\text{gt}}$ for cells $j$ are unknown. By definition of the NCS in \eqref{equ:NCS}, we have that $s_{\text{test}}=\max_{\bbx_{0:H}\in\ccalP,j\in\ccalJ_t,t\in[0:H]}1-p_t(m_j=k_j^{gt})$. Plugging this into \eqref{eq:CP1} yields the following result:
\begin{align*} 
\mathbb{P}_{{\Omega_{\text{test}}\sim\ccalD}}&\left(\max_{\bbx_{0:H}\in\ccalP,j\in\ccalJ_t,t\in[0:H]}(1-p_t(m_j=k_j^{gt})\leq\hat{s} \mid \bbm_{0:H})\right)\\&\geq1-\alpha.
\end{align*}
Thus, for any path $\bbx_{0:H}\in\ccalP$, we have that 
\begin{align}\label{eq:CP2}
\mathbb{P}_{{\Omega_{\text{test}}\sim\ccalD}}&\biggl(\bigwedge_{\forall j\in\ccalJ_t,\forall t\in[0:H],}p_t(m_j=k_j^{gt})\geq1-\hat{s} \mid \bbm_{0:H}\biggr)\nonumber\\
&\geq1-\alpha.
\end{align}
Recall from \eqref{equ:gridPredSet} that the prediction sets $\ccalC_j(\bbm_t)$ collect all labels $k\in\ccalK$ satisfying $p_t(m_j=k)\geq 1-\hat{s}$. Combining this with \eqref{eq:CP2} results in \eqref{eq:predSetProperty}, i.e., $\mathbb{P}_{{\Omega_{\text{test}}\sim\ccalD}}\left(C_{0:H}|\bbm_{0:H}\right)\geq1-\alpha$, where $C_{0:H}=\bigwedge_{\forall j\in\ccalJ_t,t\in[0:H]}(k_j^{\text{gt}}\in\ccalC_j(\bbm_t))$.
\end{proof}

{Observe that the result in Proposition~\ref{prop:perception} does not rely on Assumption~\ref{assume:sensors}, as it focuses only on map cells $j \in \mathcal{J}_t$ for which at least one measurement has been taken.}
%
Building on Proposition \ref{prop:perception}, we show that the path $\bbx_{0:H}$ generated by Alg.~\ref{algo:planner} addresses Problem \ref{prob:main}. To show this result, we make the following assumptions about the exploration and planning algorithms employed in Alg.~\ref{algo:planner}.

\begin{assumption}\label{assume:exploreFinite}
(i) The exploration algorithm is invoked a finite number of times before a path $\bbx_{0:H}$ is computed. Each invocation reduces mapping uncertaity-and consequently prediction set size-so that the planner switches back to exploitation mode. 
%
(ii) The planning algorithm used in the exploitation phase is complete, i.e., if a feasible path $\bbp$ exists that completes $\psi$ under the assigned labels described in Section \ref{sec:safePlan}, then the planner is guaranteed to find it. 
{These assumptions ensure} that mission failures can only be attributed to mapping errors. 
\end{assumption}

\begin{theorem}[Planning Guarantees]\label{thm:theorem}
    {Consider a test environment $\Omega_{\text{test}}\sim\ccalD$ with unknown true labels $\{k_j^{\text{gt}}\}_{j\in\ccalJ}$, and a mission specification $\psi_{\text{test}}$. Under Assumptions \ref{assume:sensors} and \ref{assume:exploreFinite}, Alg. \ref{algo:planner} will generate a path $\bbx_{0:H}$ satisfying \eqref{eq:pr}.}
\end{theorem}

\begin{proof}
%
First, by Assumption \ref{assume:exploreFinite}, in any {environment $\Omega_{\text{test}} \sim \ccalD$, and for any task $\psi_{\text{test}}$}, Algorithm \ref{algo:planner} is guaranteed to terminate and return a path $\bbx_{0:H}$. Consequently, any failure to complete $\psi_{\text{test}}$ can be attributed only to mapping errors. 
Second, consider the path $\bbx_{0:H}$ returned by Algorithm \ref{algo:planner} for {the task $\psi_{\text{test}}$, when operating in $\Omega_{\text{test}}\sim\ccalD$}. We show that $\bbx_{0:H}$ satisfies \eqref{eq:pr} by leveraging Proposition \ref{prop:perception}. Specifically, by construction of the exploration and exploitation modes in Algorithm \ref{algo:planner}, Assumption \ref{assume:exploreFinite}-(ii), and  Assumption \ref{assume:sensors} we have that: (i) all states $\bbx_t$ in $\bbx_{0:H}$ satisfy the imposed safety constraints given the labels $k_j=\argmax_{k\in\ccalC_j(\bbm_t)} d_k$ assigned to cells $j$, i.e., $\max_{k\in\ccalC_j(\bbm_t)} d_k+r_m\leq\lVert\bbx_{t+1}|_\Omega-c_j\rVert,\forall j\in\ccalJ_t$ and (ii) $\bbx_H\in\ccalX_g$. Now, let us assume that the ground truth label is included in the prediction sets for all $t\in\{0,\ldots,H\}$. This equivalently means that the probability that $\bbx_{0:H}$ satisfies $\psi$ with respect to the assigned labels, and implicitly the ground truth labels, is $1$. The implicit satisfaction with respect to the ground truth is due to the consideration of the worst-case label in the prediction set that also contains the ground truth. We express this as follows: 
\begin{align}\label{equ:pathSatisfy}
\mathbb{P}_{{\Omega_{\text{test}}\sim\ccalD}}\left(\bbx_{0:H}\models\psi|C_{0:H},\bbm_{0:H}\right)=1,
\end{align}where $C_{0:H}=\bigwedge_{\forall j\in\ccalJ_t,t\in[0:H]}\left(k_j^{gt}\in\ccalC_j(\bbm_t)\right)$. 
Recall from Proposition \ref{prop:perception}, that $\mathbb{P}_{{\Omega_{\text{test}}\sim\ccalD}}\left(C_{0:H}|\bbm_{0:H}\right)\geq1-\alpha$. Combining this with \eqref{equ:pathSatisfy} we get that $\bbx_{0:H}$ satisfies \eqref{eq:pr}. \end{proof} 

\begin{remark}[Dataset-conditional guarantees]\label{rem:datasetConditional}
The guarantees in Proposition \ref{prop:perception} and Theorem \ref{thm:theorem} are marginal, where the probabilities are over the randomness of the calibration set and the test environment. A dataset-conditional guarantee which holds for a fixed calibration set across environments can also be applied \cite{vovk2012conditional}. In such cases, Proposition \ref{prop:perception} is modified such that with a probability of at least $1-\delta$ over the sampling of the calibration set{$\{\Omega_i\}_{i=1}^D\sim\ccalD$}, 
    $\mathbb{P}_{{\Omega_{\text{test}}\sim\ccalD}}\left(\ccalC_{0:H}|\bbm_{0:H},{\{\Omega_i\}_{i=1}^D}\right)\geq Beta^{-1}_{D+1-v,v}(\delta),$ where $Beta^{-1}_{D+1-v,v}(\delta)$ is the $\delta$-quantile of the Beta distribution with parameters, $D$ - the size of the calibration dataset - and $v=\lfloor(D+1)\hat{\alpha}\rfloor$ and $\hat{\alpha}$ is selected so that $Beta^{-1}_{D+1-v,v}(\delta)\approx 1-\alpha$. 
    Similarly, Theorem \ref{thm:theorem} is modified, where with a probability of at least $1-\delta$ over the sampling of the calibration dataset, $\mathbb{P}_{{\Omega_{\text{test}}\sim\ccalD}}\left(\bbx_{0:H}\models\psi|\bbm_{0:H},{\{\Omega_i\}_{i=1}^D}\right)\geq Beta^{-1}_{D+1-v,v}(\delta)$.
\end{remark}

\section{Experimental Validation}\label{sec:sims}
In this section, we empirically validate the proposed approach. Sections \ref{sec:setup} introduces the setup of our experiments. 
%
The empirical performance of our method is presented in {Sections \ref{sec:results}-\ref{sec:NumPaths}}, while Section \ref{sec:compare} provides comparative results against uncertainty-informed and uncertainty-agnostic planners. 
Finally, Sections {\ref{sec:hardware}}-\ref{sec:ood} evaluate our algorithm in out-of-distribution (OOD) {real-world} and simulation scenarios, highlighting potential failure cases.
{Demonstrations in simulation and real-world are provided in \cite{int_com_video}.} 

\vspace{-0.25cm}
\subsection{Setting Up Experiments and the Proposed Planner}\label{sec:setup}
\textbf{Robot:} Our experiments are conducted using ROS~2 Humble \cite{ros2} and the Gazebo (Ignition) simulator. The simulated platform is a Clearpath Husky skid-steer ground vehicle with a maximum linear speed of $1.0~\text{m/s}$ and a maximum yaw rate of $0.5~\text{rad/s}$. It is equipped with a forward-facing RGB-D camera with a resolution of $640\times480$, a horizontal FoV of $90^\circ$, a vertical FoV of $73.7^\circ$, and a sensing range of $r_{\max}=10$~m. Note that this sensor is not omnidirectional, as required in Assumption~\ref{assume:sensors}, due to its FoV; consequently, our algorithm is tested in a more challenging setting. 



\begin{figure}[t]
\centering
\includegraphics[width=0.9\linewidth]{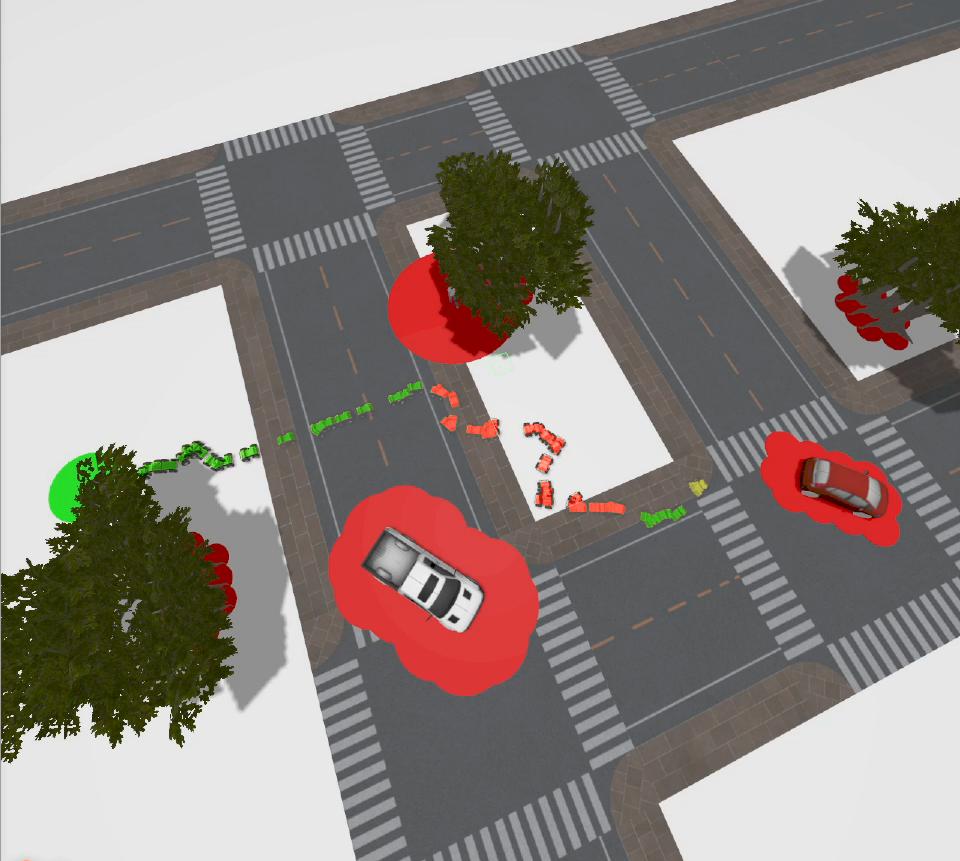}
\vspace{-0.3cm}
\caption{Illustration of a path $\bbx_{0:H}$ to the goal region (green disk) consisting of an exploration sub-path (red) and two exploitation sub-paths (green). The red regions model the unsafe regions as per \eqref{equ:pathSatisfy}.}
\vspace{-0.5cm}
\label{fig:environment}
\end{figure}

\textbf{Distribution $\ccalD$:} The distribution $\ccalD$ is designed to generate {environments} as follows. All environments $\Omega$ have dimensions of $70 \times 25$~m containing objects from five semantic classes $\mathcal{K} = \{\texttt{person}, \texttt{car}, \texttt{truck}, \texttt{tree}, \texttt{free}\}$.
Each environment is populated with one car, one truck, one human placed at randomized positions and orientations, with the human biased toward being placed near trees, and three parallel rows of three trees  in fixed locations.
The class-based safety constraints are: $d_{\text{person}}=4m$, $d_{\text{car}}=1m$, $d_{\text{truck}}=2m$, and $d_{\text{tree}}=0.5m$.
This choice enables testing our algorithm in cases of partial occlusions.
%
%
An example environment is shown in Figure~\ref{fig:environment}.

\textbf{Semantic Mapping:} To set up the semantic mapping algorithm introduced in \cite{ssmi_temp}, we use a Detectron2 model \cite{detectron2}, trained on the COCO dataset \cite{cocodataset}, to semantically segment RGB images. The segmented outputs are fused over time to construct a 3D semantic octomap with resolution $r_m=1$~m. In our experiments, since we focus on a ground robot, we project this 3D map onto the X–Y plane, while considering height ranges relevant to the robot; 
this gives rise to a 2D semantic class–based grid map. In doing so, we assume that no semantic objects are stacked vertically. 

\textbf{Exploitation-Exploration:} The exploitation mode of Alg.~\ref{algo:planner} is implemented using a standard $\texttt{A}^*$ search algorithm on the 2D grid map \cite{a-star}. For exploration, we design an algorithm that randomly selects a next state $\bbx_{t+1}$ near a voxel with a non-singleton prediction set subject to \eqref{equ:newSafetyConstraint}, aiming to decrease its prediction set size. 
As discussed in Section~\ref{sec:safePlan}, however, any other exploration algorithm can be substituted.\footnote{{A potential approach to handle cases in which the exploration algorithm fails to sufficiently reduce environmental uncertainty within a reasonable time is to find the largest $1-\alpha' < 1-\alpha$ for which a feasible path exists. This reduces the success rate to $1-\alpha'$ but allows the proposed method to generate a best-effort path.}}


\textbf{Conformalization of Maps:} Next, we discuss our CP implementation described in Section \ref{sec:calibration}. Recall that the NCS in \eqref{equ:NCS} requires evaluating the maximum over the set $\ccalP$, which contains an infinite number of paths. Since this is impractical, following \cite{PwC}, we approximate the NCS using a finite subset $\hat{\mathcal{P}}\subset\ccalP$ consisting of 10 paths {between start and goal locations that are randomly selected in the calibration environment $\Omega_i\sim\ccalD$. At test time, given $\Omega_{\text{test}}\sim\ccalD$, we randomly select the start and goal positions for the task $\psi$}. Also, to avoid generating a new calibration dataset for every $\Omega_{\text{test}}$, we apply CP as described in Remark \ref{rem:datasetConditional}, with $\delta=0.1$. 
Our calibration dataset consists of $50$ environments. 

\textbf{Evaluation Metrics:} We consider four metrics. (i) \textit{Mapping success rate}, defined as {the percentage of cases, where the following binary variable is $1$:} $S.R_{\text{map}}:=\mathds{1}(C_{0:H}),$ where $C_{0:H} = \bigwedge_{\forall j\in\ccalJ_t, t\in[0:H]} \left(k_j^{\text{gt}} \in \ccalC_j(\bbm_t)\right)$ and $k_j^{\text{gt}}$ is the ground truth label of cell $j$. In words, for a given test {environment $\Omega_{\text{test}}$}, $S.R_{\text{map}} = 1$ if all constructed prediction sets $\ccalC_j(\bbm_t)$ contain the true labels $k_j^{\text{gt}}$ for all $j \in \ccalJ_t$ and $t \in [0:H]$. Due to Proposition \ref{prop:perception}, the average mapping success rate of Alg. \ref{algo:planner} across test environments drawn from $\ccalD$ is expected to be at least $1-\alpha$ {even in the presence of a sensor with limited FoV}. (ii) \textit{Mission success rate}, defined as {the percentage of cases, where the  binary variable} $S.R_{\text{mission}}:=\mathds{1}\left(\bbx_{0:H}\models\psi\right)$ is $1$. 
In words, given a test {environment $\Omega_{\text{test}}$, and task $\psi_{\text{test}}$}, $S.R_{\text{mission}} = 1$ if the robot path $\bbx_{0:H}$ satisfies the safety constraints in \eqref{equ:pathSatisfy} and $\bbx_H\in\ccalX_g$. {Note that if the sensor was omnidirectional, the mission success rate of Alg. \ref{algo:planner} should have been to be at least $1-\alpha$ due to Theorem \ref{thm:theorem}; see our discussion in Section \ref{sec:results}.} 
(iii) \textit{Distance} traveled by the robot during deployment, which captures efficiency of the generated paths. (iv) \textit{Proportion of the overall path used for exploration}, {denoted as $p_{\text{expl}}$}, quantifying how much of the trajectory was dedicated to reducing uncertainty in the environment. {We report the average of these four metrics over 61 test environments. We also report the standard deviation of the binary variables $S.R_{\text{map}}$ and $S.R_{\text{mission}}$.
} 

\subsection{Empirical Evaluation of the Proposed Planner}\label{sec:results}

We evaluate our proposed method for three user-defined success rates: $1-\alpha \in\{ 95\%, 90\%, 85\%\}$ over $61$ environments drawn from $\ccalD$. 
%
The empirical results are reported in Table \ref{tab:empirical}.
{Observe that the empirical mapping success rates validate Prop. \ref{prop:perception}.}
%
It can be seen that for $1-\alpha=95\%$, the empirical mapping success rate is $93.36\%<95\%$. Such small deviations are due to the finiteness of the calibration dataset, the number of paths that are considered for calibration, and the number of validation environments; see \cite{CPbase} for more information on the influence of finiteness on the empirical success rates. 
We also observe that the empirical mission success rates validate Theorem \ref{thm:theorem}.
{
We note that we did not observe any mission failures due to limited FoV. 
\footnote{{A heuristic approach to prevent mission failures due to limited FoV is to have the robot spin at time $t$ to collect measurements, or to restrict motions into unknown areas lying within its blind spots.}}}
%
The path length increases with higher values of $1-\alpha$. This is expected, as higher $1-\alpha$ typically results in non-singleton prediction sets, which produce a more conservative assignment of labels to cells, as discussed in Section \ref{sec:safePlan}. Consequently, our planner generates longer paths to account for the most conservative labels. Similarly, since higher $1-\alpha$ often leads to non-singleton prediction sets, the proportion of the overall path spent on exploration is also higher.

\begin{table}[t]
    \centering
    \caption{Comparative Evaluations}
    \begin{tabular}{cccccc}
        \toprule
        Planner & \makecell{$1-\alpha$ \\ (\%)} & \makecell{$S.R_{\text{map}}$ \\ (\%)} &
        \makecell{$S.R_{\text{mission}}$ \\ (\%)} &
        \makecell{Path\\length \\ ($m$)} &
        \makecell{$p_{\text{expl}}$ \\ (\%)} \\
        \midrule
        & 95 & 93.36 $\pm$ 3.19 & 98.36 $\pm$ 1.63 & 74.7 & 19\\
        Ours & 90 & 90.16 $\pm$ 3.81 & 96.72 $\pm$ 2.28 & 71.7 & 18.54\\
        & 85 & 86.89 $\pm$ 4.32 & 95.08 $\pm$ 2.77 & 69.72 & 11.2\\
        \midrule
        & 95 & 58.33 & 93.33 & 69.66 & 8.17 \\
        \makecell{UI}& 90 & 41.67 & 76.67 & 66.14 & 4.65 \\
        & 85 & 35 & 75 & 62.53 & 5.04 \\
        \midrule
        \makecell{UA} & NA & 10.17 & 65.57 & 48.91 & NA \\
        \bottomrule
    \end{tabular}
    \label{tab:empirical}
\end{table}

\subsection{Effect of Number $|\hat{\ccalP}|$ of Calibration Plans} \label{sec:NumPaths}

{As mentioned in Section \ref{sec:setup}, we use a finite set $\hat{\ccalP}\subset\ccalP$ to estimate the NCS in $\eqref{equ:NCS}$. In what follows, we evaluate the performance of our algorithm for different sizes of $\hat{\ccalP}$. Specifically, we use the same $50$ calibration environments and the same $61$ test environments described in Section \ref{sec:setup} for $|\hat{\ccalP}| = 1, 5, 10,$ and $20$. 
Notably, the empirical mapping success rates shown in Table \ref{tab:size of calibration paths} remain close to $1 - \alpha$ for all considered $|\hat{\ccalP}|$, while the standard deviation of $S.R_{\text{map}}$ tends to decrease, as $|\hat{\ccalP}|$ increases. We emphasize that this observation may be case-specific; in theory, the mapping success rate is guaranteed to be at least $1 - \alpha$ only when $\hat{\ccalP} = \ccalP$. As $|\hat{\ccalP}|$ increases, the NCS becomes larger, resulting in larger, non-singleton prediction sets. This increased conservativeness is reflected in higher (though not necessarily monotonic) mission completion rates, longer paths, and larger exploration proportion.  
}

\begin{table}[t]
    \centering
    \caption{Effect of $|\hat{\ccalP}|$}
    \begin{tabular}{cccccc}
        \toprule
        \makecell{$1-\alpha$ \\ (\%)} & \makecell{$n$} & \makecell{$S.R_{\text{map}}$ \\ (\%)} & \makecell{$S.R_{\text{mission}}$ \\ (\%)} &
        \makecell{Path\\length \\ ($m$)} &
        \makecell{$p_{\text{expl}}$ \\ (\%)} \\
        \midrule
        & 1 & 90.33 $\pm$ 6.04 & 96.56 $\pm$ 2.34 & 63.727 & 5.59\\
        & 5 & 87.66 $\pm$ 4.12 & 89.22 $\pm$ 3.94 & 75.206 & 13.59\\
        90 & 10 & 90.16 $\pm$ 3.81 & 96.72 $\pm$ 2.28 & 71.7 & 18.54\\
        & 20 & 90 $\pm$ 3.79 & 100 $\pm$ 0 & 76.7338 & 18.81\\
        \bottomrule
    \end{tabular}
    \label{tab:size of calibration paths}
\end{table}

\begin{figure}[t]
\centering
\includegraphics[width=\linewidth]{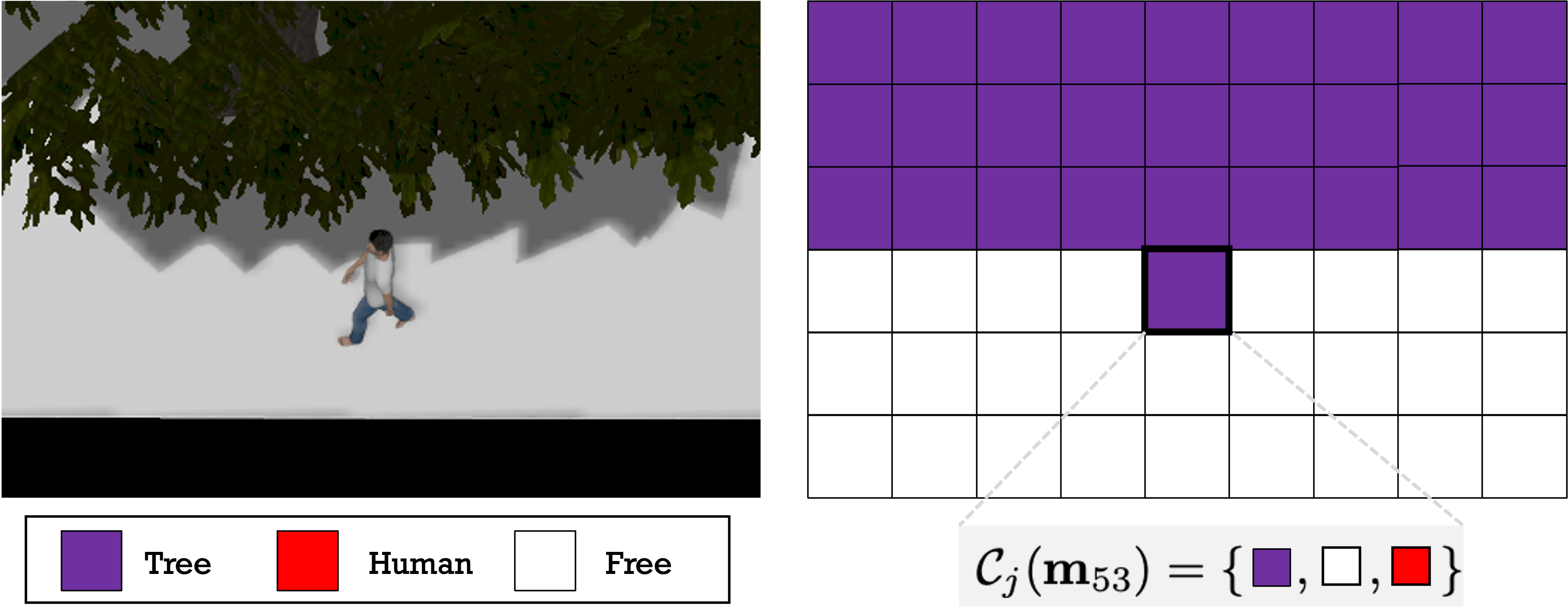}
\vspace{-0.5cm}
\caption{Top view of the environment (left) and the corresponding map $\bbm_t$ (right). The most likely label in $\bbm_t$ for cell $j$, occupied by a human, is `tree’. In contrast, the prediction set constructed by our method for that cell mitigates this mapping error by including `tree’, `human’, and `free-space’. 
}
\vspace{-0.4cm}
\label{fig:occlusion}
\end{figure}

\subsection{Comparative Experiments}\label{sec:compare}

\textbf{Baselines:} We compare our planner against the following baselines. (i) \textit{Uncertainty-agnostic (UA) Planner:} This planner assigns to each cell $j$ the most-likely label as per the PMFs $p_t(m_j)$, i.e., the label assigned to cell $j$ at time $t$ is $k_j=\argmax_{k\in\ccalK}p_t(m_j)$. Using this label assignment, the planner employs A* to compute paths that complete $\psi$. We select this baseline to illustrate the importance of designing planners that are aware of mapping uncertainty. (ii) \textit{Uncertainty-informed (UI) Planner:} The second baseline differs from our planner only in the construction of prediction sets. Specifically, the predictions sets are built heuristically by treating  the PMFs $p_t(m_j)$ as true probabilities. Thus, the prediction sets are constructed as $ \ccalC_j(\bbm_t)=\{\pi_1,\ldots,\pi_k\},$
where $k=\sup\left\{k':\sum_{i=1}^{k'}p_t(m_j=\pi_i)<1-\alpha\right\}+1$, and $\pi$ is the permutation of $\ccalK$ that sorts the classes from most likely to least likely, i.e. $p_t(m_j=\pi_{n-1})>p_t(m_j=\pi_n)$. We select this baseline to illustrate the importance of calibrating uncertainty metrics provided by mapping algorithms.

\textbf{Comparative Results:} 
Table \ref{tab:empirical} shows the performance of both baselines using the evaluation metrics discussed in Section \ref{sec:setup}. To ensure a fair comparison, both baselines are evaluated on the same $61$ test {environments, and their corresponding reach-avoid tasks} used for our proposed method. Notice that the UI planner outperforms the UA planner in terms of mission and mapping success rates, as it accounts for mapping uncertainty. However, our method outperforms the UI baseline because the UI prediction sets—unlike ours—are constructed heuristically. 

As seen in Table \ref{tab:empirical}, our approach leads to a significant improvement in mapping and mission success rates. Our empirical evaluation shows that most baseline failures occur due to partial occlusions in the environment. An example is illustrated in Figure \ref{fig:occlusion}, where a human is partially occluded by trees. The cell associated with the human is often labeled as tree' by the UA planner, leading to safety violations because the human' label has stricter safety constraints than the tree' label. In this case, by the time the UA planner detects the human, a safety constraint has already been violated. Similarly, the UI planner builds a singleton prediction set containing only the `tree' label. In contrast, our approach constructs a prediction set that includes `tree', `human', and `free-space', enabling the planner to follow a conservative path and maintain safety. This results in higher mission and mapping success rates, though at the cost of longer paths, as shown in Table \ref{tab:empirical}.





\begin{figure}[t]
\centering
\includegraphics[trim={0 0 0 0},clip,width=0.85\linewidth]{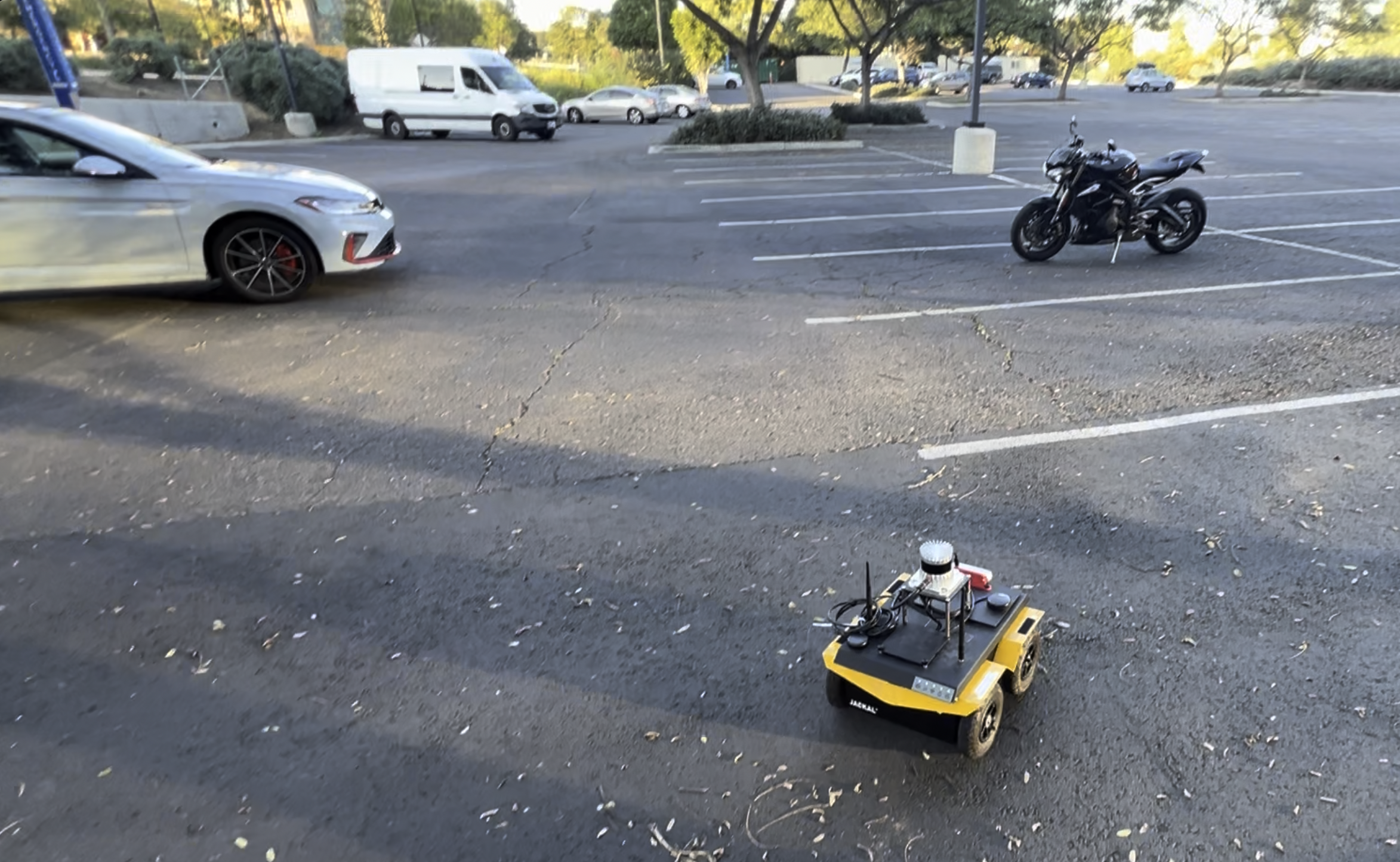}
\vspace{-0.3cm}
\caption{{Hardware demonstration of the proposed method.} 
}\vspace{-0.6cm}
\label{fig:realworld}
\end{figure}


\subsection{Hardware Experiments}\label{sec:hardware}

{In this section, we demonstrate robot behaviors with respect to various values for $1-\alpha$ in real-world settings. Specifically, we consider a Clearpath Jackal equipped with an Intel RealSense D455 camera and a Detectron2 segmentation model \cite{detectron2}. The robot operates in environments with humans, motorcycles, cars, and tree (see  Fig.~\ref{fig:realworld}). The class-based safety constraints are: $d_{\text{person}}=1.5m$, $d_{\text{car}}=1.25m$, $d_{\text{motorcycle}}=1m$, and $d_{\text{tree}}=1m$. Similar to the simulation results, we observed an increasing trend in path length as $1-\alpha$ increases. Specifically, the total traveled distance for $1-\alpha = 90\%, 95\%, \text{and } 99\%$ was $24.38$~m, $24.92$~m, and $25.78$~m, respectively. We note that our method was applied using calibration data described in Section~\ref{sec:setup}, which induces a distribution shift. Evaluation of our approach in OOD settings in terms of mapping and mission success rates is provided in Section~\ref{sec:ood}. }

\subsection{Evaluation in Out-of-Distribution (OOD) Scenarios }\label{sec:ood}
{Finally, we evaluated our method in OOD environments. For calibration, we use the same $50$ calibration environments sampled from $\ccalD$ mentioned in Section \ref{sec:setup}, with $|\hat{\ccalP}|=10$. In what follows, we set $1-\alpha=90\%$. We consider the case where the test environments are sampled from a distribution $\ccalD'$ where, unlike $\ccalD$, the trees vary in number from $1$ to $10$, and are placed randomly. In addition, the environments also simulated the lighting conditions of dusk and dawn instead of midday used in $\ccalD$. We observed a drop in mapping success rates for $1-\alpha=90\%$ from $90.16\%\pm3.81\%$ to $80.33\%\pm5.09\%$. The mission success rates also dropped from $96.72\%\pm2.28\%$ to $93.44\%\pm3.17\%$, albeit still being greater than $1-\alpha$. We empirically observed that changes in the number and spatial distribution of the trees have a more significant impact than the lightning conditions on the performance of the algorithm.}

\section{Conclusions and Future Work}\label{sec:conc}
In this paper, we presented a novel semantic planner for unknown environments. Unlike related works, our method designs probabilistically correct paths in the presence of environmental uncertainty, stemming from imperfect perception, without any assumptions on sensor models or noise characteristics. 
Future work will focus on: (i) designing active exploration algorithms that directly target reducing the size of prediction sets; (ii) accounting for OOD scenarios by leveraging robust CP methods; and (iii) directly quantifying sensor (rather than mapping) uncertainty. 



\bibliographystyle{IEEEtran}
\bibliography{YK.bib, sources.bib}
\end{document}